\documentclass[conference]{IEEEtran}

\usepackage{microtype}
\usepackage{graphicx}
\usepackage{tabularx}
\usepackage{multicol}
\usepackage{booktabs} % for professional tables
\usepackage{subfig}

\usepackage{amsmath}
\usepackage{amssymb}
\usepackage{mathtools}
\usepackage{amsthm}
\usepackage{amsfonts}

\usepackage{algpseudocode}
\usepackage{algorithm}
\usepackage{multirow}
\usepackage{array}
\usepackage{xcolor}
\usepackage{mathtools}
\usepackage{amsthm}
\usepackage{bbm}
\usepackage{wrapfig,lipsum,booktabs}
\usepackage{newfloat}
\usepackage{listings}
\usepackage{bm}
\usepackage{microtype}
\usepackage{caption}
\usepackage{url}

\newcommand{\N}{\mathcal{N}}

\newtheorem{theorem}{Theorem}

\newtheorem{definition}{Definition}

\AtBeginDocument{%
  \providecommand\BibTeX{{%
    \normalfont B\kern-0.5em{\scshape i\kern-0.25em b}\kern-0.8em\TeX}}}
    
\begin{document}

\title{Stochastic Variational Inference with\\ Tuneable Stochastic Annealing}

\author{
\IEEEauthorblockN{John Paisley}
\IEEEauthorblockA{\textit{Columbia University}}
\and
\IEEEauthorblockN{Ghazal Fazelnia}
\IEEEauthorblockA{\textit{Spotify Research}}
\and
\IEEEauthorblockN{Brian Barr}
\IEEEauthorblockA{\textit{Capital One Labs}}
}

\maketitle

\begin{abstract}
We exploit the observation that stochastic variational inference (SVI) is a form of annealing and present a modified SVI approach --- applicable to both large and small datasets --- that allows the amount of annealing done by SVI to be tuned. We are motivated by the fact that, in SVI, the larger the batch size the more approximately Gaussian is the noise of the gradient, but the smaller its variance, which reduces the amount of annealing done to escape bad local optimal solutions. We propose a simple method for achieving both goals of having larger variance noise to escape bad local optimal solutions and more data information to obtain more accurate gradient directions. The idea is to set an actual batch size, which may be the size of the data set, and an \textit{effective} batch size that matches the increased variance of a smaller batch size. The result is an approximation to the maximum entropy stochastic gradient at a desired variance level. We theoretically motivate our ``SVI+'' approach for conjugate exponential family model framework and illustrate its empirical performance for learning the probabilistic matrix factorization collaborative filter (PMF), the Latent Dirichlet Allocation topic model (LDA), and the Gaussian mixture model (GMM).
\end{abstract}

\section{Introduction}
Posterior inference is one of the core problems of Bayesian modeling. There are two fundamental approaches to posterior inference: One uses Markov chain Monte Carlo (MCMC) sampling techniques, which are asymptotically correct, but tend to be slow compared to point-estimate techniques and not scalable to large datasets \cite{Hastings70}. Variational inference is a second approach that approximates the posterior distribution with a simpler distribution family and then learns its parameters by minimizing their Kullback–Leibler (KL) divergence \cite{Jordan-VI-99}. This optimization imposes new types of challenges, for example finding closed form expressions or scaling to large datasets.

To deal with non-convexity in variational inference, methods such as selective marginalization \cite{fazelnia2022probabilistic} or convex relaxations have been proposed \cite{fazelnia2018crvi}. Annealing methods have also shown promise at improving local optimal solutions. Annealing has been studied for variational inference in a variety of contexts \cite{katahira2008deterministic,yoshida2010bayesian,gultekin2015stochastic,BleiVT}. These approaches often perform a deterministic inflation of the entropy term in the variational objective to allow for more exploration of the parameter space in early iterations, gradually shrinking to the true entropy over time to optimize the desired KL divergance. Stochastic variational inference is another implicit, less investigated annealing approach that we consider in this paper.

Stochastic variational inference (SVI) is a practical method that allows for efficient inference over large datasets \cite{hoffman2013stochastic}. It combines natural gradients and stochastic optimization to learns the variational parameters over an entire dataset quickly. As with other stochastic gradient descent methods, SVI's gradients are random because the subset of data they are calculated over is sampled, and the variance of this gradient is inversely proportional to the number of subsamples selected in each iteration. This introduces an annealing effect \cite{mandt2016variational}, which has benefits beyond inference for big data when the noise is Gaussian \cite{welling2011bayesian}. However, if Gaussian noise is desired in the stochastic gradient of SVI a catch-22 arises: To achieve a more Gaussian noise in the gradients a larger batch size is required, but larger batch sizes also reduce the variance of the gradient, potentially reducing the annealing effect. As a result, the algorithm may still not be able to escape bad local optimal solutions. 

The goal of our paper is to simultaneously increase the Gaussianity of the stochastic gradients by increasing the batch size, while not reducing the desired effective variance of this batch size. We do this using mathematical properties of our considered model framework itself rather than the simpler but potentially less practical trick of directly adding Gaussian noise to the gradients. 
% We propose a simple, but mathematically motivated method to increase the Gaussianity of the stochastic gradients produced by SVI by simultaneously increasing the batch size, while reintroducing the variance that is naturally lost in the process. 
In the context of conjugate exponential family models, the result is the introduction of a \textit{one-dimensional} noise term that can be included in the stochastic update of the natural parameter. The variance of this noise is determined by the desired effective batch size as defined below. This desired effective sample size is less than the actual batch size used in the stochastic update. We derive the general algorithm, and illustrate its ability to find better local optima on three models: the probabilistic matrix factorization collaborative filter, the Latent Dirichlet Allocation topic model, and the Gaussian mixture model. We argue that this methodology can benefit both stochastic and batch inference, bringing the advantages of SVI to smaller data sets as well.

The paper is organized as follows: Section 2 provides a review of the conjugate exponential family model structure we focus on, and the general stochastic variational inference technique for that model framework. We then derive and present our proposed method in Section 3 for increasing the annealing performance of SVI while still using larger batch sizes. The result is a very simple general purpose algorithm with no additional computational overhead. In Section 4 we illustrate the approach with empirical results on three popular models for machine learning that highlights different aspects of the methods potential usefulness, including improving batch VI by incorporating big data inference ideas in smaller problems.

\section{Background}\label{sec.model.VI}
We derive our modification to SVI, called SVI+ (``plus'' more annealing), in the context of a single variable within a probabilistic graphical model. Assume a set $X = (x_1,\dots,x_N)$, where each $x_n$ constitutes a single observation, such as a document, image, or a feature vector.
Many generative models can be written in form 
\begin{equation}
 x_n\,|\,\theta_n \stackrel{ind}{\sim} p(X|\theta_n,\beta),\quad  \theta_n \stackrel{iid}{\sim} p(\theta),\quad \beta \sim p(\beta).
\end{equation}
Here, the \textit{local} variables $\theta_n$ impact data $x_n$, while the \textit{global} variables $\beta$ are used by all of the data. The joint likelihood of this model can be written
\begin{equation}
  p(x,\theta,\beta) = p(\beta)\prod\nolimits_{n=1}^N p(x_n|\theta_n,\beta)p(\theta_n).
\end{equation}
When the full posterior $p(\theta,\beta|x)$ is intractable, we can approximate it with a factorized distribution $q(\beta,\theta) = q(\beta)\prod_n q(\theta_n)$ using variational inference. The variational objective function here takes the form
\begin{eqnarray}
 \mathcal{L} &=&   \mathbb{E}_q[\ln p(x,\theta,\beta)] - \mathbb{E}_q[\ln q(\beta,\theta)]\nonumber\\[5pt]
 &=&    \mathbb{E}_q\Big[\ln \frac{p(\beta)}{q(\beta)}\Big] + \sum\nolimits_{n=1}^N \mathbb{E}_q\Big[\frac{\ln p(x_n,\theta_n|\beta)}{q(\theta_n)}\Big].
\end{eqnarray}
We assume that the model is in the conjugate exponential family (CEF). As a result, each $q$ distribution is optimally chosen to be in the same exponential family as the prior, and has a natural parameter vector: $\lambda_n$ for $q(\theta_n)$ and $\lambda$ for $q(\beta)$.

Since optimizing $q(\theta_n)$ does not change from standard VI inference in the SVI setting, we focus on the stochastic update for $q(\beta)$, which requires updating $\lambda$. Let $\eta$ be the natural parameter for the prior distribution on $\beta$. In the standard batch variational inference approach, the natural parameter update of $q(\beta)$ for a CEF model equals
\begin{equation}
   \lambda = \eta + \sum\nolimits_{n=1}^N \mathbb{E}_q[t(\theta_n)],
\end{equation}
where the expectation uses the current value of $\lambda_n$ in $q(\theta_n)$ and $t(\theta_n)$ is the sufficient statistic function of the conditional posterior distribution of $\beta$. The variational update therefore sums over the expectation of these sufficient statistics. Next, we discuss how these updates are modified for stochastic inference, and note how there is an implied Gaussian annealing being done by SVI. That discussion will point towards a modified SVI algorithm presented later, which we call SVI+.

\subsection{Stochastic Variational Inference}
Stochastic variational inference (SVI) can be used when $N$ is too large to update each $q(\theta_n)$ in a reasonable amount of time within one iteration. In this setting, each iteration samples a subset of $x_n$, indexed by set $\mathcal{S}_t$ at iteration $t$. It then constructs the iteration-dependent objective function
\begin{equation}
   \mathcal{L}_t = \mathbb{E}_q\Big[\ln \frac{p(\beta)}{q(\beta)}\Big] + \frac{N}{|\mathcal{S}_t|}\sum\nolimits_{n\in \mathcal{S}_t} \mathbb{E}_q\Big[\frac{\ln p(x_n,\theta_n|\beta)}{q(\theta_n)}\Big].
\end{equation}
Then, the stochastic update $\lambda \leftarrow (1-\rho_t)\lambda + \rho_t\nabla_{\lambda}\mathcal{L}_t$ is made, which in the CEF framework is equal to
\begin{equation}\label{eq.stochasticupdate}
   \lambda \leftarrow (1-\rho_t)\lambda + \rho_t\Big(\eta + \frac{N}{|\mathcal{S}_t|}\sum\nolimits_{n\in \mathcal{S}_t} \mathbb{E}_q[t(\theta_n)]\Big).
\end{equation}
The step size $\rho_t > 0$ satisfies $\sum_{t=1}^{\infty} \rho_t = \infty$ and $\sum_{t=1}^{\infty}\rho_t^2 < \infty$. The expectation of this gradient over the randomness of the subset $\mathcal{S}_t$ at iteration $t$ equals the ground truth gradient over the entire dataset. Combined with the conditions on $\rho_t$, this guarantees convergence to a local optimal solution of $\mathcal{L}$.

\subsection{SVI as an Annealing Method}
We can think of the stochastic update in (\ref{eq.stochasticupdate}) as
\begin{equation}\label{eq.impliednoise}
   \lambda \leftarrow (1-\rho_t)\lambda + \rho_t\Big(\eta + \sum\nolimits_{n=1}^N \mathbb{E}_q[t(\theta_n)] + \epsilon_t\Big),
\end{equation}
where $\epsilon_t \sim p_t(\epsilon)$ is generated from an implied zero-mean distribution. Therefore, in addition to a scalable inference method, SVI can also be viewed as an annealing method. That is, by subsampling the data, SVI implicitly adds noise to the true gradients, which is known in other optimization contexts to have the potential advantage of escaping bad local optimal solutions. Incidentally, this helps explain why SVI can be observed to sometimes find better local solutions than batch VI for optimizing the same variational objective function.

Simulated annealing methods often use Gaussian noise, while theoretical analysis of annealing using Langevin dynamics also presupposes Gaussian noise \cite{welling2011bayesian}, and so it is useful to understand how Gaussian is the implied random $\epsilon_t$ introduced by SVI in (\ref{eq.impliednoise}). To this end, we observe that we can write the actual stochastic update of $\lambda$ in (\ref{eq.stochasticupdate}) as
\begin{equation}
\begin{aligned}\label{eqn.stochgrad}
   &\lambda \leftarrow (1-\rho_t)\lambda + \rho_t(\eta + N \lambda'_t),\\[5pt]
    &\lambda'_t = \frac{1}{|\mathcal{S}_t|}\sum\nolimits_{n\in \mathcal{S}_t} \mathbb{E}_q[t(\theta_n)].
\end{aligned}
\end{equation}
The vector $\lambda'_t$ is the average of the expected sufficient statistics over the random subsample $\mathcal{S}_t$ of iteration $t$. The terms in this subsample are chosen $iid$ from the empirical data distribution. Therefore, by the central limit theorem, this vector is asymptotically Gaussian
\begin{equation}\label{eq.clt}
  \lambda'_t \stackrel{d}{\longrightarrow} \N(\mu,\Sigma_t),
\end{equation}
where the mean and covariance matrix are
\begin{equation}
   \mu =  \frac{1}{N}\sum\nolimits_{n=1}^N \mathbb{E}_q[t(\theta_n)],\quad \Sigma_t = \frac{1}{|\mathcal{S}_t|} \mathrm{Cov}(\mathbb{E}_q[t(\theta)]).
\end{equation}
The expected value of $\lambda'_t$ is the true gradient because the stochastic gradients of SVI are unbiased. The covariance of $\mathbb{E}_q[t(\theta)]$ is calculated over the empirical distribution of the entire data set. Therefore, the noise of SVI is approximately Gaussian, and the larger the batch size $\mathcal{S}_t$, the more Gaussian it is by the CLT.
While Gaussian noise helps anneal the variational objective, increasing the Gaussianity by increasing $|\mathcal{S}_t|$ also provides less stochasticity in the SVI update, since $\Sigma_t$ is shrinking. 
We next propose a method for achieving both aims of, 1) having more approximate Gaussian noise, and 2) higher variance gradients, to find better local optima.

\section{SVI+ for More Annealing with SVI}\label{sec.SVI+}
We propose a modification of SVI, which we call SVI+ (``plus'' more annealing). We first discuss the general mathematical framework, which while technically able to be performed directly, is greatly simplified after with a small approximation. The resulting SVI+ algorithm requires little modification to an existing SVI algorithm and virtually no computational overhead. The final SVI+ algorithm is shown in Algorithm \ref{alg.svi+}.

\subsection{Mathematical Motivation for the Proposed Method}
We propose simply adding back the variance to the stochastic gradients lost when $|\mathcal{S}_t|$ increases as follows,
\begin{equation}\label{eq.stochmath}
   \lambda \leftarrow (1-\rho_t)\lambda + \rho_t\big(\eta + N(\mu + \epsilon_t + \xi_t)\big).
\end{equation}
$\mu$ is the true batch gradient from (\ref{eq.clt}), and $\epsilon_t$ and $\xi_t$ are two noise vectors. The distributions of these vectors are
\begin{equation}
\begin{aligned}
& \epsilon_t ~~\stackrel{d}{\longrightarrow} ~\, \N\big(0,|\mathcal{S}_t|^{-1} \mathrm{Cov}(\mathbb{E}_q[t(\theta)])\big),\\
& \xi_t ~~\,\,\sim\,\,~~ \N(0,\alpha\mathrm{Cov}(\mathbb{E}_q[t(\theta)])).
\end{aligned}
\end{equation}
The first random vector $\epsilon_t$ is the noise implied by SVI using batch size $|\mathcal{S}_t|$ as previously discussed; $\mu + \epsilon_t$ is the SVI gradient. The second noise vector $\xi_t$ is our proposed addition for SVI+ and $\alpha>0$ is a parameter discussed later. From the perspective of the two noise vectors, their addition $\epsilon_t + \xi_t$ is also asymptotically Gaussian, with 
\begin{equation}\label{eq.effectivevariance}
   \epsilon_t + \xi_t \stackrel{d}{\longrightarrow}  \N\big(0,(\alpha+|\mathcal{S}_t|^{-1})\mathrm{Cov}(\mathbb{E}_q[t(\theta)])\big).
\end{equation}
As $|\mathcal{S}_t|$ increases this approximation is more accurate, but since the \textit{effective} variance of the noise is $(\alpha+\frac{1}{|\mathcal{S}_t|})\mathrm{Cov}(\mathbb{E}_q[t(\theta)])$ because of the addition of $\xi_t$, the actual variance can be controlled by a lower bound parameter $\alpha$. 

What is the use of having a lower bound in $\alpha$ when the batch size $|\mathcal{S}_t|$ is a parameter of SVI that can be selected to achieve any value of $\alpha$ for the stochastic gradient? Here we emphasize the difference between the approximate Gaussianity of $\epsilon_t$, which improves with increasing $|\mathcal{S}_t|$ by the CLT, and the actual Gaussianity of $\xi_t$ for all values of $\alpha$ as defined above. In practice, when $|S_t|$ is small, the actual noise distribution of the stochastic gradient is likely far from Gaussian. The purpose of our proposed SVI+ method is to have larger variance (as defined by the lower bound $\alpha$) with more Gaussian noise (as achieved with increasing $|\mathcal{S}_t|$). Our experiments will provide empirical evidence that, between two unbiased stochastic gradients sharing the same covariance, the one that is Gaussian is preferred. This also happens to be the maximum entropy gradient, as discussed later.

Therefore, the parameter $\alpha$ allows us to define an \textit{effective variance} for our stochastic gradient. We define this in terms of a smaller target batch size. First, note that $|\mathcal{S}_t|$ is the \textit{actual} batch size used in SVI+. Let $M$ be the smaller \textit{effective} batch size whose larger variance we seek to match with a Gaussian noise vector. We can do this simply by tuning $\alpha$:
\begin{equation}
   \frac{1}{M} = \alpha + \frac{1}{|\mathcal{S}_t|}\quad \Rightarrow \quad \alpha = \frac{|\mathcal{S}_t|-M}{M|\mathcal{S}_t|}.
\end{equation}
The LHS of the first equation is the desired variance for (\ref{eq.effectivevariance}) that approximates a batch of size $M$. However, for small $M$ the \textit{actual} SVI update is far from Gaussian, while for $|\mathcal{S}_t| > M$ and $\alpha$ defined as the above, the covariance is identical to a batch of size $M$, but closer to Gaussian distributed. We therefore have an \textit{effective} stochastic batch size of $M$ while using an actual batch size of $|\mathcal{S}_t|$.

Setting $\alpha$ to mimic a batch of size $M$, and returning to (\ref{eq.stochmath}) and writing out the stochastic gradient $\mu+\epsilon_t$ used in optimization, we can rewrite the update as follows,
$$  \lambda \leftarrow (1-\rho_t)\lambda + \rho_t\Big(\eta + \frac{N}{|\mathcal{S}_t|}\sum\nolimits_{n\in \mathcal{S}_t} \mathbb{E}_q[t(\theta_n)] + N\xi_t\Big),$$
\begin{equation}\label{eqn.up1}
\xi_t \sim \N\Big(0,\frac{|\mathcal{S}_t|-M}{M|\mathcal{S}_t|}\mathrm{Cov}(\mathbb{E}_q[t(\theta)])\Big).
\end{equation}
In Section \ref{sec.analysis} we discuss the SVI+ gradient of (\ref{eqn.up1}) in terms of maximum entropy and analyze its convergence to a Gaussian as a function of $|\mathcal{S}_t|$ and $M$. First, we modify (\ref{eqn.up1}) for implementation; since $\mathrm{Cov}(\mathbb{E}_q[t(\theta)])$ is difficult to obtain in practice and may be too high dimensional to directly work with, we discuss a small approximation that simplifies the implementation of (\ref{eqn.up1}) by not requiring this matrix.

\begin{algorithm}[t]
\caption{SVI+ for Tuneable Annealing of VI}\label{alg.svi+}
\begin{algorithmic}[1]
\Require Batch size $|\mathcal{S}_t|$ and effective batch size $M < |\mathcal{S}_t|$
\Function{At iteration $t$}{}\vspace{2pt}
\State \textbf{Sample} data indices $\mathcal{S}_t \subset \{1,\dots,N\}$\vspace{2pt}
\State \textbf{Sample} $\varepsilon_n \sim_{iid} \N(0,|\mathcal{S}_t|/M - 1)$ for $n \in \mathcal{S}_t$\vspace{2pt}
\State \textbf{Calculate} sample average $\overline{\varepsilon}$ over all $\varepsilon_n$ ($\approx 0$)\vspace{2pt}
\State \textbf{Construct} noise-added statistic vector $$\textstyle\lambda'_t = \frac{1}{|\mathcal{S}_t|}\textstyle\sum_{n\in \mathcal{S}_t} (1+\varepsilon_n-\overline{\varepsilon})\mathbb{E}_q[t(\theta_n)]$$
\State \textbf{Calculate} annealed stochastic update
$$\textstyle\lambda \leftarrow (1-\rho_t)\lambda + \rho_t(\eta + N\lambda'_t)$$
 \textbf{return} $\lambda$ --- Note: SVI special case when $M = |\mathcal{S}_t|$\EndFunction
\end{algorithmic}
\end{algorithm}

\subsection{Approximation and Simplification}
The value of $\mathrm{Cov}(\mathbb{E}_q[t(\theta)])$ in (\ref{eqn.up1}) is unknown in practice because calculating it would defeat the purpose of scalable inference. We propose an approximation that has the advantage of greatly simplifying the final algorithm, making the proposed method very easy to implement and requiring effectively no additional computational resources compared with SVI. Specifically, we suggest using the data in batch $\mathcal{S}_t$ to approximate the true $\mathrm{Cov}(\mathbb{E}_q[t(\theta)])$ calculated over the entire dataset:
\begin{equation}
  \mathrm{Cov}(\mathbb{E}_q[t(\theta)]) ~\approx~ \Sigma_t
\end{equation}
where we define
\begin{equation}\label{eqn.musigparams}
\begin{aligned}
   \Sigma_t & = \frac{1}{|\mathcal{S}_t|} \sum\nolimits_{n\in \mathcal{S}_t} (\mathbb{E}_q[t(\theta_n)] - \mu_t)(\mathbb{E}_q[t(\theta_n)] - \mu_t)^\top,\\
   \mu_t & = \frac{1}{|\mathcal{S}_t|} \sum\nolimits_{n\in \mathcal{S}_t} \mathbb{E}_q[t(\theta_n)] .
\end{aligned}
\end{equation}
A simplification results from this approximation by observing that we can generate the random target vector $\xi_t$ having the desired approximate covariance $\Sigma_t$, 
\begin{equation}\label{eqn.target}
     \xi_t \sim \N\Big(0,\frac{|\mathcal{S}_t|-M}{M|\mathcal{S}_t|}\Sigma_t\Big),
\end{equation}
as follows, 
\begin{equation}\label{eqn.gen1}
\begin{aligned}
   \xi_t &= \bigg(\frac{|\mathcal{S}_t|}{M}-1\bigg)^{\frac{1}{2}}\frac{1}{|\mathcal{S}_t|}\sum\nolimits_{n\in \mathcal{S}_t} (\mathbb{E}_q[t(\theta_n)] - \mu_t) \varepsilon_n,\\
   \varepsilon_n &\stackrel{iid}{\sim} \N(0,1).
\end{aligned}
\end{equation}
To verify that (\ref{eqn.target}) and (\ref{eqn.gen1}) are equivalent: \textit{i}\,) $\xi_t$ in (\ref{eqn.gen1}) must be Gaussian by construction, \textit{ii}\,) $\mathbb{E}[\xi_t] = 0$ because $\mathbb{E}[\varepsilon_n]=0$, and \textit{iii}\,) $\mathrm{Cov}[\xi_t]=\mathbb{E}[\xi_t\xi_t^\top] = \frac{|\mathcal{S}_t|-M}{M|\mathcal{S}_t|} \Sigma_t$ because $\mathbb{E}[\varepsilon_n^2] = 1$.
Inputting $\mu_t$ from (\ref{eqn.musigparams}) into (\ref{eqn.gen1}), we observe  that
\begin{eqnarray}
   \sum\nolimits_{n\in \mathcal{S}_t}\mu_t\varepsilon_n  &=&  \sum\nolimits_{n\in \mathcal{S}_t}\mathbb{E}_q[t(\theta_n)]\overline{\varepsilon},\nonumber\\
   \overline{\varepsilon} &=& \frac{1}{|\mathcal{S}_t|}\sum\nolimits_{n\in \mathcal{S}_t} \varepsilon_n.
\end{eqnarray}
We can therefore write the vector in (\ref{eqn.gen1}) in the following equivalent way that is useful for practical implementation:
\begin{equation}\label{eqn.gen2}
   \xi_t = \bigg(\frac{|\mathcal{S}_t|}{M}-1\bigg)^{\frac{1}{2}}\frac{1}{|\mathcal{S}_t|}\sum\nolimits_{n\in \mathcal{S}_t} \mathbb{E}_q[t(\theta_n)] (\varepsilon_n-\overline{\varepsilon}),
\end{equation}
$$   \varepsilon_n \stackrel{iid}{\sim} \N(0,1), \quad
   \overline{\varepsilon} = |\mathcal{S}_t|^{-1}\sum\nolimits_{n\in \mathcal{S}_t} \varepsilon_n.
$$
The original SVI+ update was given in (\ref{eqn.up1}). After replacing $\xi_t$ in that equation with the approximation in (\ref{eqn.gen2}), then absorbing $\sqrt{|\mathcal{S}_t|/M - 1}$ in the variance term of $\varepsilon_n$ and simplifying, we obtain the stochastic update defined below.

\begin{definition}[SVI+] Let $\mathcal{S}_t \subset \{1,\dots,N\}$ be the index set of the batch used at iteration $t$ for SVI. Let $M < |\mathcal{S}_t|$ be the desired effective batch size whose stochastic gradient is to be approximately noise-matched with a Gaussian distribution. For an exponential family $q$ distribution in a conjugate exponential family model, the SVI+ update for the natural parameters in iteration $t$ is $\lambda \leftarrow (1-\rho_t)\lambda + \rho_t\big(\eta + \lambda'_t\big)$, where the stochastic gradient $\lambda'_t$ is calculated as follows:
\begin{eqnarray}
\lambda'_t &=& \frac{N}{|\mathcal{S}_t|}\sum\nolimits_{n\in \mathcal{S}_t} (1+\varepsilon_n-\overline{\varepsilon})\mathbb{E}_q[t(\theta_n)],\nonumber\\
\varepsilon_n &\stackrel{iid}{\sim}& \N\big(0,|\mathcal{S}_t|/M - 1\big),\quad n \in \mathcal{S}_t\nonumber\\
\overline{\varepsilon} &=& \frac{1}{|\mathcal{S}_t|}\sum\nolimits_{n\in \mathcal{S}_t} \varepsilon_n.\nonumber
\end{eqnarray}
\end{definition}

The difference between this algorithm and the original SVI update in (\ref{eq.stochasticupdate}) is that random variables $\varepsilon_n - \overline{\varepsilon}$ are included as multipliers of the expected sufficient statistics of the observations. We notice that, when $M = |\mathcal{S}_t|$, $\varepsilon_n = 0$ with probability one and so we recover SVI. We also observe that SVI+ may be useful for small-scale problems as well. In that case, we let $\mathcal{S}_t = \{1,\dots,N\}$ for every iteration. We then form the batch update, but add randomness according to the effective batch size $M$ for annealing the variational objective and hopefully converging to a better local optimal solution.

\subsection{Discussion: Max Entropy, CLT Convergence \& an Extension}\label{sec.analysis}
Let $X\sim\mathcal{P}$ be an arbitrary continuous random variable with $\mathbb{E}[X] = \mu$ and $\mathrm{Cov}(X) = \Sigma$. Then it is well-known that the entropy of $\mathcal{P}$ is upper-bounded by a Gaussian distribution $\mathcal{N}$ having the same mean and covariance: Since the Kullback-Leibler divergence KL$(\mathcal{P}\|\mathcal{N}) \geq 0$, it follows that
\begin{equation}
    \mathcal{H}(\mathcal{P}) \,\leq  -\int \mathcal{P}(x)\ln \mathcal{N}(x) \, dx  \,=\,  \mathcal{H}(\mathcal{N}),
\end{equation}
where $\mathcal{H}$ is entropy. If $\lambda'_{t}$ is a stochastic gradient of batch size $|\mathcal{S}|$ at iteration $t$, then the distribution $\lambda'_{t} \sim \mathcal{P}$ is likely far from Gaussian when $|\mathcal{S}|$ is small --- recall that $\mathcal{P}$ approaches the empirical distribution of the gradients as $|\mathcal{S}|\rightarrow 1$ and converges to a Gaussian as $|\mathcal{S}|\rightarrow\infty$ due to the CLT. 
For an \textit{effective} batch size $M<|\mathcal{S}|$, SVI+ replaces the true distribution on stochastic gradient $\mathcal{P}$ for batch size $|\mathcal{S}|$ with a distribution $\mathcal{Q}$ that is more approximately Gaussian and has the same mean and covariance as $\mathcal{P}$ when $|\mathcal{S}|=M$. Therefore, SVI+ takes stochastic gradient steps approximating the maximum entropy distribution over batches of size $M$ by using $|\mathcal{S}|>M$ samples. The following theorem gives the rate of convergence of the SVI+ gradient distribution to maximum entropy.

\begin{theorem}[Max Entropy Gradients]
Let $M$ be the effective batch size of SVI+ and write the actual batch size as $|\mathcal{S}| = \tau M$, $\tau \geq 1$, where $\mathcal{S}$ is an index set selected uniformly iid from $\{1,\dots,N\}$. Let the SVI+ gradient $Y = \xi + |\mathcal{S}|^{-1}\sum_{n\in\mathcal{S}}\lambda_n$ where $\xi \sim \mathcal{N}\big(0,\frac{|\mathcal{S}|-M}{M|\mathcal{S}|}\Sigma_{\lambda}\big)$, and let $Z\sim \mathcal{N}(\mu_{\lambda},\frac{1}{M}\Sigma_{\lambda})$ where $\mu_{\lambda}$ and $\Sigma_{\lambda}$ are the sample mean and covariance over all $\{\lambda_n\}_{n=1}^N$. Then the total variation distance $\delta$ between the distributions on $Y$ and $Z$ is bounded by
$$\delta(\mathcal{Q}(Y),\mathcal{N}(Z)) \leq \frac{C d^{\frac{1}{4}}}{\tau^2\sqrt{M}} = \frac{C d^{\frac{1}{4}}}{\tau^{3/2}\sqrt{|\mathcal{S}|}},$$
where $d$ is the dimensionality of $Y$ and $Z$, and $C$ is a constant.
\end{theorem}
\begin{proof}[Proof (sketch)] Apply the Berry-Esseen theorem, simplify, and note the resulting expectation $\frac{1}{N}\sum_{n=1}^N \mathbb{E}_{p(\xi)}[\|\Sigma_{\lambda}^{-\frac{1}{2}}(\lambda_n + \xi)\|_2^3]$, $\xi \sim \mathcal{N}(0,\frac{1}{M}\frac{\tau-1}{\tau}\Sigma_{\lambda})$, can be bounded and absorbed in $C$.
\end{proof}

Expressing this bound in both $|\mathcal{S}|$ and $M$ via the relationship $|\mathcal{S}| = \tau M$ illustrates the two aspects of convergence in terms of these SVI+ parameters. First, when $\tau=1$, the SVI special case results and the convergence rate is $\mathcal{O}(1/\sqrt{|\mathcal{S}|})$, as expected by the CLT. For a fixed $M$, the rate is $\mathcal{O}(1/\tau^2)$ in $\tau$. The convergence is so rapid in $\tau$ because the Gaussian  $\xi$ increasingly accounts for the noise, while the noisy gradient $|\mathcal{S}|^{-1}\sum_{n\in\mathcal{S}}\lambda_n$ is itself increasingly Gaussian by the CLT.

Finally, we observe that our analysis and algorithm can be applied to any problem in which stochastic gradient descent is used. While our focus is on probabilistic models, the generic form of our objective is $\textstyle\mathcal{L} = \sum_{n=1}^N f_{\varphi}(x_n),$ for which SGD follows for learning $\varphi$. In these more general learning contexts ``SVI+'' may be usefully incorporated, e.g., in the backpropagation algorithm for learning neural networks.

\begin{figure*}[th!]
	\includegraphics[trim={9mm 2mm 9mm 2mm},clip,width=0.33\textwidth]{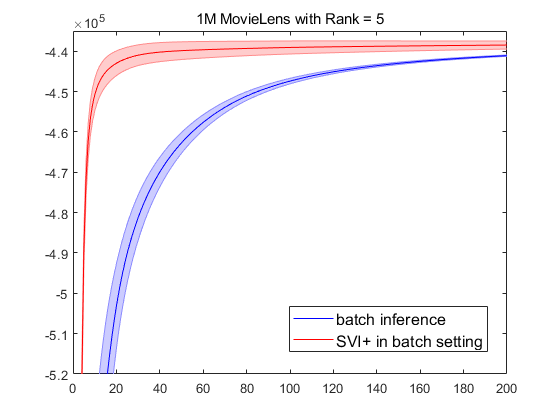}
	\includegraphics[trim={9mm 2mm 9mm 2mm},clip,width=0.33\textwidth]{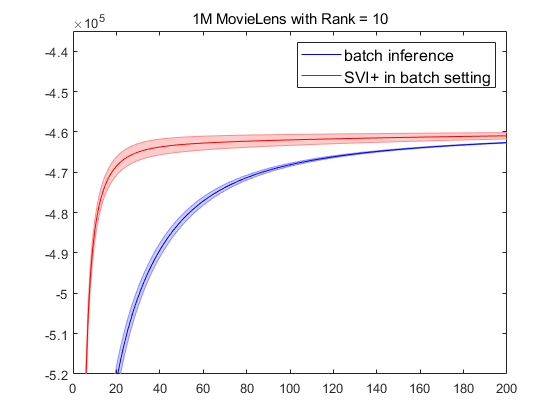}
	\includegraphics[trim={9mm 2mm 9mm 2mm},clip,width=0.33\textwidth]{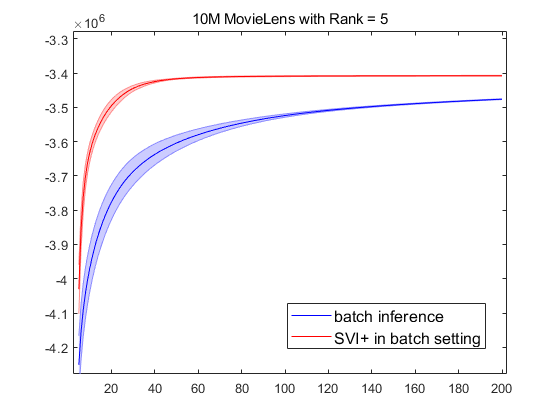}
	\caption{Variational objective as a function of iteration for the probabilistic matrix factorization collaborative filter model. We use both the 1M and 10M MovieLens data sets and average over 20 random initializations. In all experiments, the batch size $|\mathcal{S}_t|$ equals the data size, with the difference of SVI+ being the effective batch size used for stochastically annealing the gradients. SVI+ finds a better local optimal solution and converges significantly faster.}\label{fig.pmf}
\end{figure*}

\section{Empirical Results}\label{sec.experiments}
We provide some empirical results to demonstrate the ability of SVI+ to outperform both batch and stochastic variational inference in terms of finding better local optimal solutions. We demonstrate this on three models, each focusing on a different aspect of the annealing performance of SVI+.

\subsection{Example 1: Probabilistic Matrix Factorization}\label{sec.pmf}
We first consider variational inference for the probabilistic matrix factorization (PMF) model used for collaborative filtering \cite{mnih2007probabilistic}. PMF models the rating $y$ given by user $i$ to item $j$ as $y_{ij} \sim \N(u_i^\top v_j,\sigma^2)$,
for $(i,j)$ in a measured index set $\Omega$. The vectors $u,v\in\mathbb{R}^d$ parameterize the user and item locations, respectively, and have Gaussian priors. We define $q(u,v) = \prod_i q(u_i) \prod_j q(v_j)$ where each $q$ is a multivariate Gaussian distributions whose parameters are to be learned using variational inference. For the Gaussian $q(u_i)$, the sufficient statistics over the entire data set restricted to user $i$ equals
\begin{equation}
    \lambda'_{u_i} =  \sigma^{-2}\sum\nolimits_{j : (i,j)\in\Omega} \big[\mathbb{E}_q[v_j v_j^\top]\, , \,  y_{ij}\mathbb{E}_q[v_j]\big],
\end{equation}
and a similar pattern holds for each $v_j$.

\begin{figure}[bhp!]\vspace{-5pt}
 \includegraphics[trim={0mm 0mm 0mm 4.2mm},clip,width=1\columnwidth]{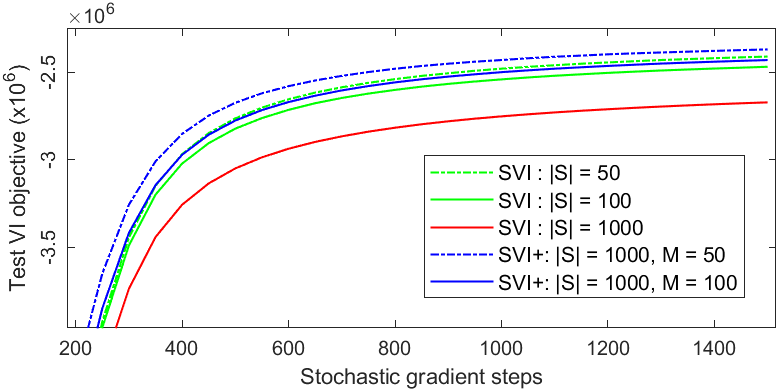}
 \caption{SVI+ compared with SVI for the LDA model averaged over 10 runs as a function of 1500 stochastic gradient steps. SVI+ with $|\mathcal{S}|=1000$ and $M=50,100$ finds better local optimal solutions than SVI using any of these batch sizes. This indicates that the annealing done by SVI+ at effective level of $M=50,100$ performs better than using SVI with an actual batch size of 50 or 100. Also observed is that increasing the batch size of SVI degrades performance, possibly due to reduced annealing from less stochasticity in the gradients.}\label{fig.LDA}
\end{figure}

We evaluate the annealing performance of SVI+ when a batch VI algorithm is modified. That is, in Algorithm \ref{alg.svi+} we take the index set $\mathcal{S}_t = \Omega$ for each iteration, but either use or don't use annealing by tuning $M$. We evaluate results on the benchmark 1 million and 10 million MovieLens data sets. We set parameters $c = 1$, $\sigma^2 = 1/2$ and note that $y$ is a star ranking between 1 and 5. We consider both $d = 5, 10$.

Our results are shown in Figure \ref{fig.pmf}, where we plot the variational objective functions from 20 randomly initialized runs. For batch inference, we use the standard VI algorithm to learn $q$. For SVI+, we use Algorithm \ref{alg.svi+} with $\mathcal{S}_t = \Omega$, $\rho = 0.85$ and an increasing $M_t = 50t$, which transitions from SVI+ to batch over the iterations. As is evident, SVI+ consistently converges to a better local optimal solution for both data sets and different factorization rank settings, which can be entirely attributed to the annealing performed by SVI that allows for better distributions to be learned on the embeddings $u$ and $v$. Batch inference is the preferred choice for this data, since each iteration of the larger 10 million set required about 9 seconds on a desktop computer. However, as Figure \ref{fig.pmf} shows, incorporating the implied annealing from SVI+ can produce a significantly better model fit with virtually no added computation time.

\begin{figure}[th!]
\centering
	\includegraphics[trim={7mm 2mm 13mm 1mm},clip,width=.49\columnwidth]{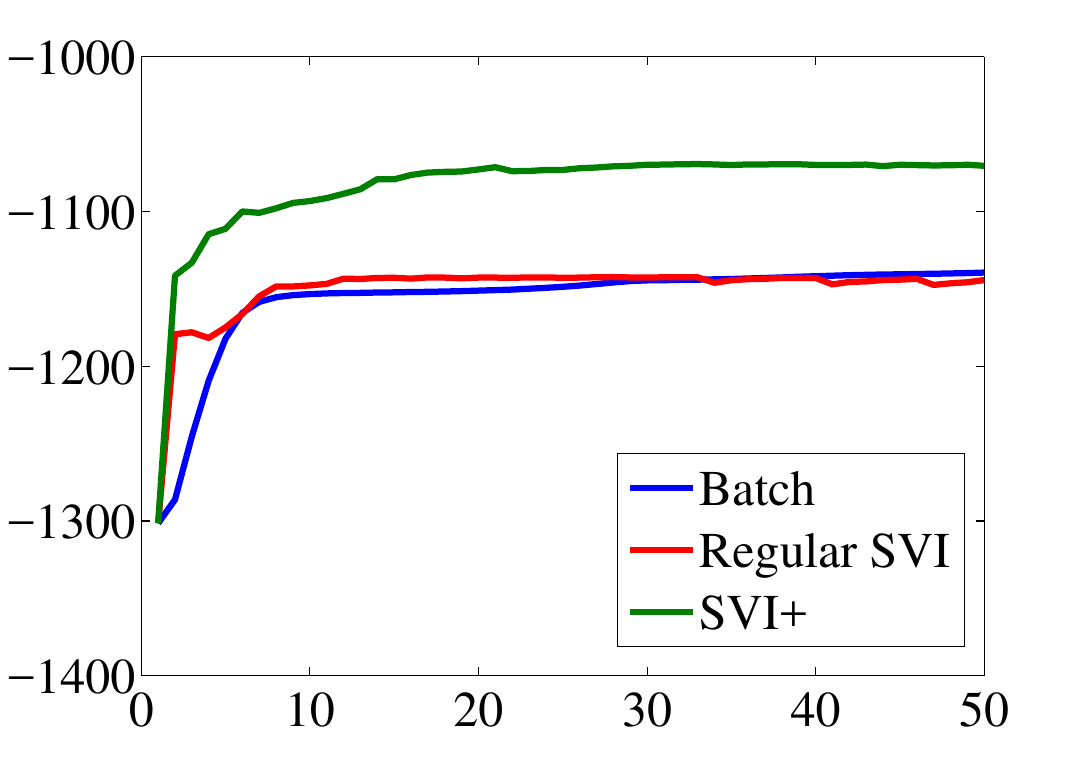}
	\includegraphics[trim={9mm 2mm 13mm 1mm},clip,width=.49\columnwidth]{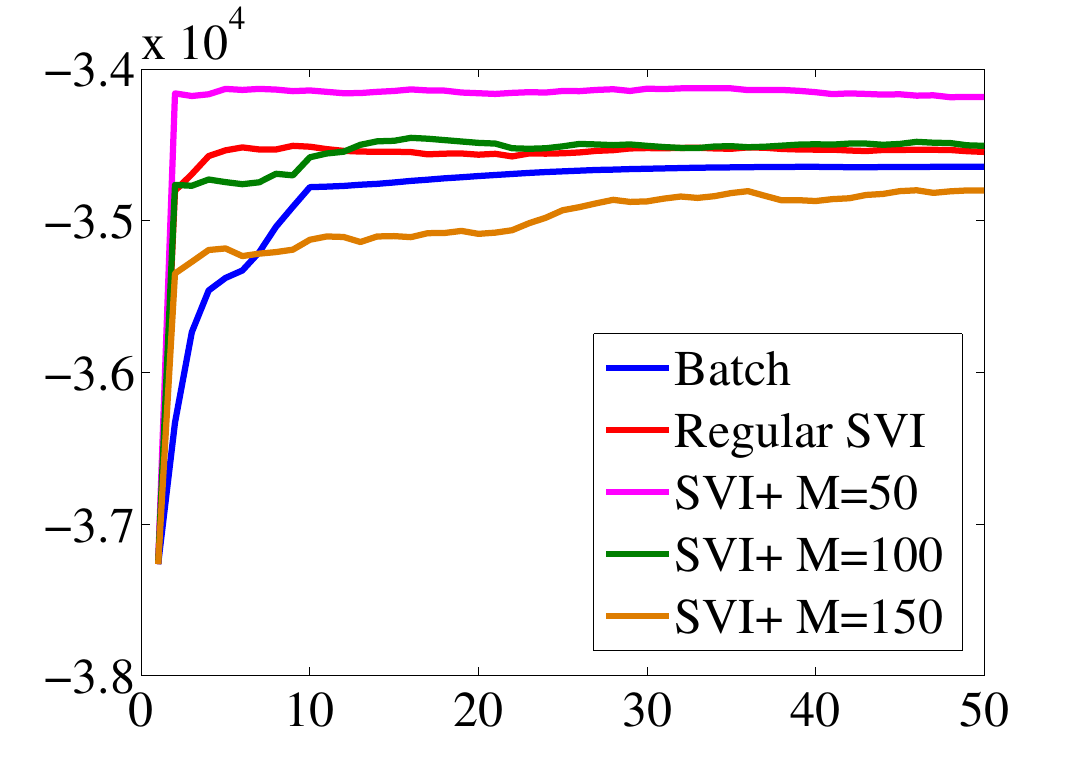}
	\caption{VI objective as a function of iteration for a typical run. Left: Synthetic data. Right: Pima dataset with four different effective batch sizes $M$.}\label{fig.obj.comp}
\end{figure}
\subsection{Example 2: Latent Dirichlet Allocation}
We also consider the topic modeling problem using Latent Dirichlet Allocation (LDA) \cite{blei2003latent}. Where the previous example demonstrated SVI+ in batch settings, here we compare how SVI+ performs against SVI when the batch size $|\mathcal{S}_t|$ of SVI is equal to either the batch size of SVI+ or the smaller $M$ setting of SVI+. We ran SVI with batch sizes $|\mathcal{S}_t| = 50, 100, 1000$ and SVI+ with $|\mathcal{S}_t| = 1000$ and $M = 50, 100$. We learn 150 topics on 286,753 articles from \textit{The New York Times} and test the VI objective on 1000 articles.

In Figure \ref{fig.LDA}, we show the variational objective as a function of batches seen averaged over 10 runs with shared randomization. As is clear, SVI+ outperforms SVI with batch size $|\mathcal{S}_t|=1000$ as a result of the increased annealing afforded by SVI+ when $M=50,100$. We also observe that SVI+ with these settings outperforms SVI when $|\mathcal{S}_t| = 50,100$. This indicates that the annealing done by SVI+ at the equivalent effective batch size $M=50,100$ when $|\mathcal{S}_t| = 1000$ is slightly better than the annealing done by SVI when $|\mathcal{S}_t| = 50,100$. Note that the improvement of SVI-50 and SVI-100 over SVI-1000 highlights the annealing done by SVI itself.

\subsection{Example 3: Gaussian Mixture Model}\label{sec.GMM}
For data $x \in \mathbb{R}^d$, the GMM generates $x_i|c_i \sim \N(\mu_{c_i},\Lambda_{c_i})$, $c_i \sim \text{Discrete}(\pi)$. Independent Dirichlet, Gaussian, and Wishart conjugate priors are used for $\pi$, $\mu_j$ and $\Lambda_j$, respectively, and posterior $q$'s are in the same family and factorization as the prior. We use a synthetic 2D dataset with 250 samples and Pima from the UCI repository, which is an 8D dataset of 768 samples. We note that these data sizes are small relative to the previous examples, demonstrating SVI+ in this setting. For synthetic data, we set $|\mathcal{S}_t|=50$ and $M=10$. For Pima, we set $|\mathcal{S}_t| = 200$ and vary $M \in \{50, 100, 150\}$. Figure \ref{fig.obj.comp} shows plots of the VI objective for the two datasets. SVI+ is able to achieve better local optimal values compared to regular SVI and batch inference. Also, SVI+ with smaller $M$ results in better objective values. 

We also consider a nonparametric extension in which the underlying number of clusters is learned by an approximate Dirichlet process (DP). In this experiment, we use synthetic data containing 4 clusters and approximate the DP with 50 clusters and a sparsity-inducing prior. We learn the model for $|\mathcal{S}_t| \in \{10,20,30,40,50\}$ and $M \in\{0.2|\mathcal{S}_t|,0.4|\mathcal{S}_t|,0.6|\mathcal{S}_t|\}$ for a total of 15 combinations. We ran each 10 times with random initializations. We show the average sorted probability distribution over clusters in Figure \ref{fig.phicomp} (top). We see that batch inference consistently learned 5 clusters while SVI learned 3 clusters. SVI+ is more accurate because it converges more consistently to a better local optima than batch because of improved annealing. Figure \ref{fig.phicomp} (bottom) shows box plots of the sorted cumulative sums of 150 runs. As can be seen, most of the data are contained in the first two clusters in SVI (blue). SVI+ (red) learns a more accurate distribution of the data.

\begin{figure}[t]
\centering
	\includegraphics[trim={22mm 5mm 13mm 5mm},clip,width=0.32\columnwidth]{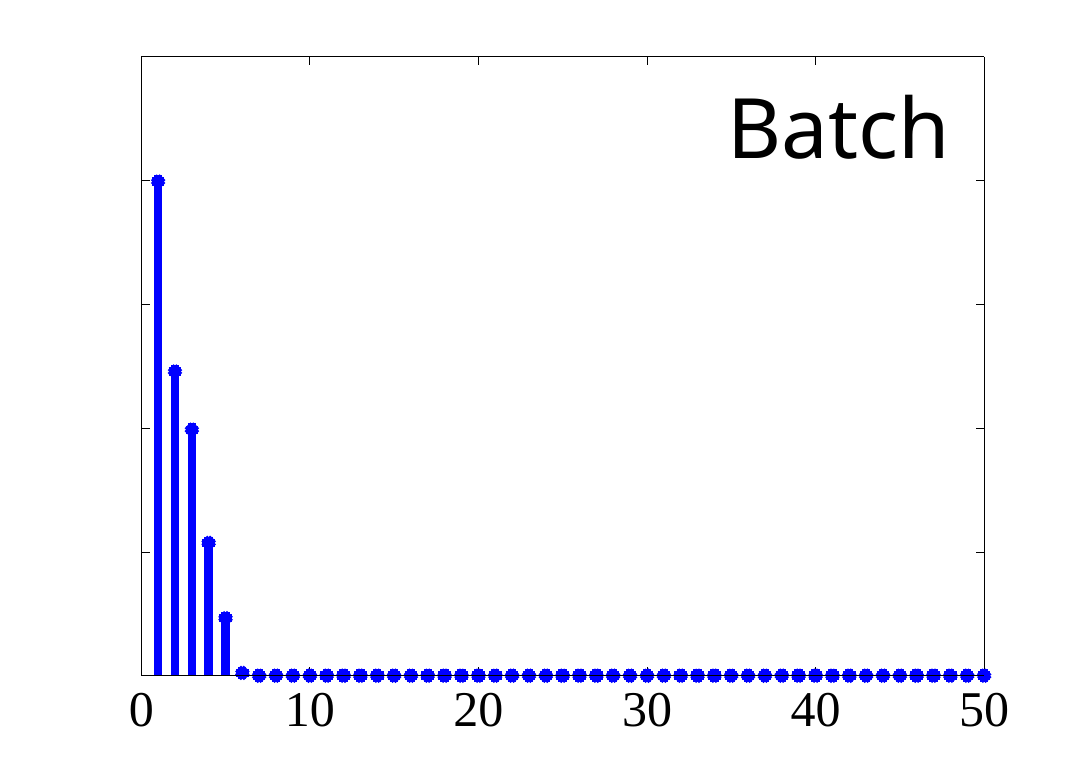}
	\includegraphics[trim={22mm 5mm 13mm 5mm},clip,width=0.32\columnwidth]{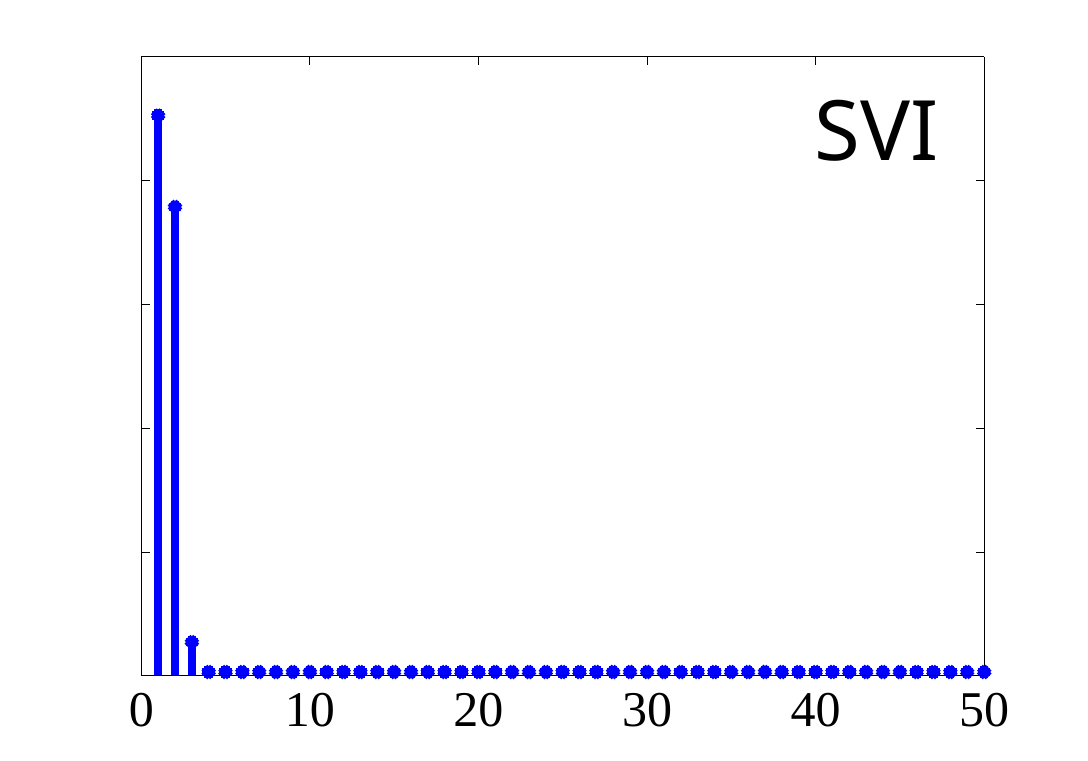}
	\includegraphics[trim={22mm 5mm 13mm 5mm},clip,width=0.32\columnwidth]{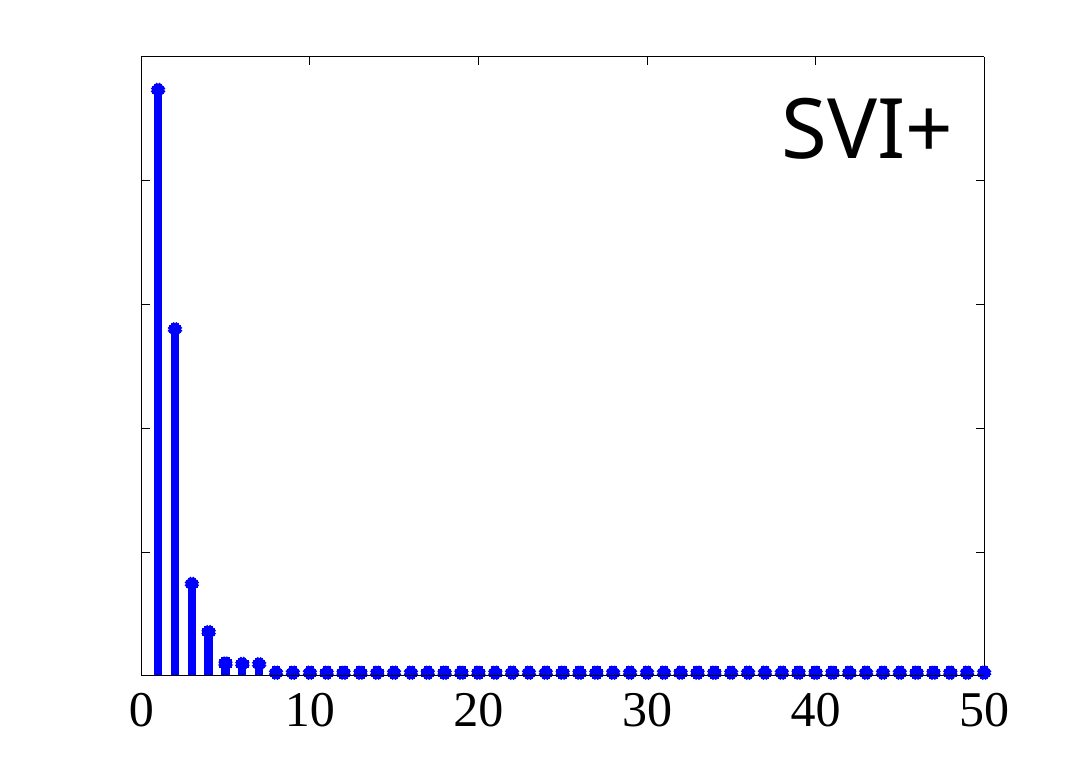}
    \includegraphics[trim={23mm 16mm 15mm 5mm},clip,width=1\columnwidth,height=1in]{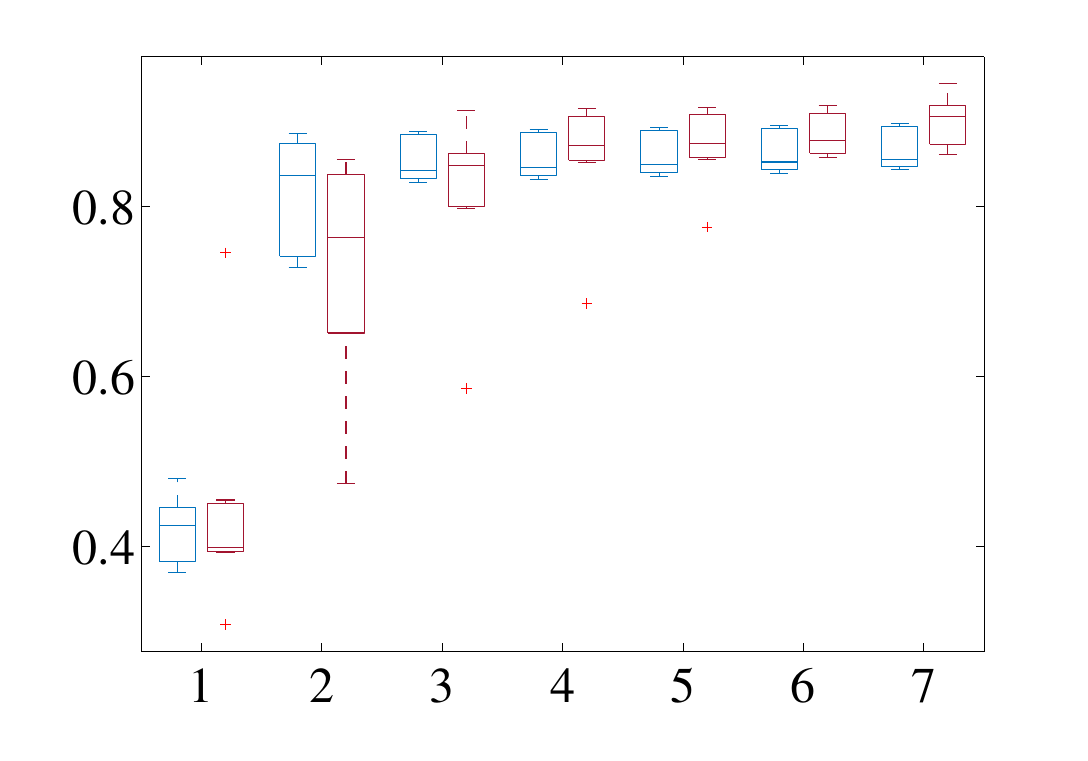}
	\caption{TOP: Average size-sorted empirical distribution of data across 50 clusters. The ground truth number of clusters was 4. Batch inference consistently over estimates this number, while SVI oversimplifies the model by underestimating it. BOTTOM: Box plot version of top (SVI vs SVI+) showing cumulative percentage of data contained up to the given cluster (150 runs). SVI (blue) puts more data in fewer clusters than SVI+ (red).}\label{fig.phicomp}
\end{figure}

\section{Conclusion}
We propose SVI+, an annealing method for VI that increases the Gaussian noise in a stochastic VI gradient by matching the variance of a larger batch to a smaller one. The large batch gives more accurate gradient information while a simple addition of 1D noise variables increases the \textit{Gaussian} noise of the gradient with no computational overhead. Experiments demonstrate the possibility for improved local optimal convergence in stochastic and batch settings when compared with standard SVI using either the corresponding larger or smaller batch size.

\bibliographystyle{IEEEtran}
\bibliography{SVIplus}

\end{document}